\documentclass[runningheads]{llncs}
\usepackage{graphicx}
\usepackage[colorinlistoftodos]{todonotes}
\usepackage{tikz}
\usepackage{etoolbox}
\usepackage{amssymb}
\usepackage{makecell}
\usepackage{enumerate}

\newtoggle{onlypaper}
\newtoggle{onlyreport}

\togglefalse{onlypaper}
\toggletrue{onlyreport}

\newcommand{\onlypaper}[1]{\iftoggle{onlypaper}{#1}{}}

\newcommand{\onlyreport}[1]{\iftoggle{onlyreport}{#1}{}}

\newcommand{\naive}{\textrm{naive}}
\newcommand{\grounded}{\textrm{grounded}}
\newcommand{\complete}{\textrm{complete}}
\newcommand{\SCOOCnaive}{\textrm{SCOOC-naive}}
\newcommand{\CFtwo}{\textrm{CF2}}
\newcommand{\SCFtwo}{\textrm{SCF2}}
\newcommand{\NSA}{\textit{NSA}}
\newcommand{\nsa}{\textrm{nsa}}
\newcommand{\Ar}{\textit{Ar}}
\newcommand{\att}{\textit{att}}
\newcommand{\SCCs}{\textit{SCCs}}
\newcommand{\scc}{\textrm{scc}}

\begin{document}
\title{SCF2 -- an Argumentation Semantics\\ for Rational Human Judgments\\ on Argument Acceptability:\\ Technical Report}

\titlerunning{SCF2 -- an Argumentation Semantics for Rational Human Judgments}
%
\author{Marcos Cramer\inst{1} 
\and Leendert van der Torre\inst{2}}
%
%
\institute{International Center for Computational Logic, TU Dresden, Germany \\
\email{marcos.cramer@tu-dresden.de} \and
Computer Science and Communications, University of Luxembourg, Esch-sur-Alzette, Luxembourg \\
\email{leon.vandertorre@uni.lu}}

\maketitle

\begin{abstract}
In abstract argumentation theory, many argumentation semantics have been proposed for evaluating argumentation frameworks. This paper is based on the following research question: Which semantics corresponds well to what humans consider a rational judgment on the acceptability of arguments? There are two systematic ways to approach this research question: A normative perspective is provided by the principle-based approach, in which semantics are evaluated based on their satisfaction of various normatively desirable principles. A descriptive perspective is provided by the empirical approach, in which cognitive studies are conducted to determine which semantics best predicts human judgments about arguments. In this paper, we combine both approaches to motivate a new argumentation semantics called SCF2. For this purpose, we introduce and motivate two new principles and show that no semantics from the literature satisfies both of them. We define SCF2 and prove that it satisfies both new principles. Furthermore, we discuss findings of a recent empirical cognitive study that provide additional support to SCF2.
\end{abstract}

\section{Introduction}
\label{sec:intro}

The formal study of argumentation is an important field of research within AI~\cite{rahwan2009argumentation}. A central focus of this field has been the idea of Dung~\cite{dung1995acceptability} that under some conditions, the acceptance of arguments depends only on a so-called {\em attack} relation among the arguments, and not on the internal structure of the arguments. Dung called this approach {\em abstract} argumentation and called the directed graph that represents the arguments as well as the attack relation between them an \emph{argumentation framework} (\emph{AF}). Whether an argument is deemed acceptable depends on the decision about other arguments. Therefore the basic concept in abstract argumentation is a {\em set} of arguments that can be accepted together, called an {\em extension}. Crucially, there may be several of such extensions, and these extensions may be incompatible. An \emph{extension-based argumentation semantics} takes as input an AF and produces as output a set of extensions.

Two classes of extension-based argumentation semantics have been studied. Dung himself introduced several examples of so-called {\em admissibility-based} semantics, formalizing the idea that an argument is acceptable in the context of an extension if the extension \emph{defends} the argument, i.e.\ attacks all the attackers of the argument. In this paper we consider his grounded, complete, preferred, and stable semantics. Moreover, we consider the  admissibility-based semantics known as semi-stable semantics~\cite{verheij1996two,DBLP:journals/logcom/CaminadaCD12}.
The other kind of extension-based argumentation semantics are {\em naive-based} semantics, which are based on the idea that acceptable arguments sets are specific maximal conflict-free sets. In this paper we consider the naive, stage, CF2 and stage2 semantics and develop a new naive-based semantics called SCF2.


Abstract argumentation has various potential applications \cite{rahwan2009argumentation}, and the choice of the semantics depends on the envisioned application. In this paper, we focus on the following research question: Which semantics corresponds well to what humans consider a rational judgment on the acceptability of arguments? 

There are two systematic ways to approach this research question: A normative perspective is provided by the \emph{principle-based approach} \cite{baroni2007principle}, in which semantics are evaluated based on their satisfaction of various normatively desirable principles. A descriptive perspective is provided by the \emph{empirical approach} \cite{rahwan2010behavioral}, in which cognitive studies are conducted to determine which semantics best predicts human judgments about arguments. In this paper, we combine both approaches. 

Two recent empirical cognitive studies on argumentation semantics by Cramer and Guillaume \cite{cramer2018empirical,cramer2019empirical} showed CF2 to be better predictors of human argument evaluation than admissibility-based semantics like grounded and preferred. This finding sheds some doubt on principles that are only satisfied by admissibility-based semantics, e.g.\ Admissibility, Defence and Reinstatement as defined by van der Torre and Vesic \cite{vanderTorre18}. 
For this reason, in this paper we focus on another existing principle, namely Directionality, and introduce two new ones.


The first new principle we consider is \emph{Irrelevance of Necessarily Rejected Arguments} (\emph{INRA}). Informally, INRA says that if an argument is attacked by every extension of an AF, then deleting this argument should not change the set of extensions. The idea here is that an argument that is attacked by every extension would be rejected by any party in a debate, and hence would never be brought up in a debate. Hence, it should be treated as if it did not even exist.

The second principle that we consider is \emph{Strong Completeness Outside Odd Cycles} (\emph{SCOOC}). Informally, SCOOC says that if an argument $a$ and its attackers are not in an odd cycle, then an extension not containing any of $a$'s attackers must contain $a$. The principle is based on the idea that it is generally desirable that an argument that is not attacked by any argument in a given extension should itself be in that extension. While it is possible to ensure this 
property in AFs without odd cycles, this is not the case for AFs involving an odd cycle. The idea behind the SCOOC principle is to still satisfy this property as much as possible, i.e.\ whenever the argument under consideration and its attackers are not in an odd cycle. 

We show that of the nine common semantics mentioned above, the only ones that satisfy INRA are grounded, complete and naive semantics, Additionally, we show that a variant of CF2 that we call $\nsa(\CFtwo)$ and that consists of first deleting all self-attacking arguments and then applying CF2 semantics also satisfies INRA. 

Furthermore, we show that of these ten semantics (the nine mentioned at the beginning as well as $\nsa(\CFtwo)$), the only one that satisfies SCOOC is the stable semantics. But stable semantics satisfies neither Directionality nor INRA.
The fact that none of the considered existing semantics satisfies both new principles introduced in this paper raises the question whether these two principles can be satisfied in conjunction. We answer this question positively by defining a novel semantics called \emph{SCF2 semantics} that satisfies both of them.

Finally, we discuss findings of a recent cognitive study by Cramer and Guillaume \cite{cramer2019empirical} whose results suggest that SCF2 is more in line with the judgments of participants than any existing semantics. So our hypothesis that SCF2 corresponds well to what humans consider a rational judgment on the acceptability of arguments is motivated not only by theoretical but also by empirical observations. The robustness of these preliminary empirical findings will need to be tested in future studies.

\onlypaper{All proofs of theorems in this paper can be fond in a technical report \cite{SCF2technicalreport}.}

\section{Preliminaries}
\label{sec:prelim}

In this section we define required notions from abstract argumentation theory \cite{dung1995acceptability,Baroni18}. 
Additionally, we define three principles from the literature on principle-based argumentation \cite{baroni2007principle,vanderTorre18} and present an argument for the case that the Directionality principle is a desirable property for a semantics designed to match what humans would consider a rational judgment on the acceptability of arguments.


\begin{definition}
 An \emph{argumentation framework (AF)} $F = \langle \Ar,\att \rangle$ is a finite directed graph in which the set $\Ar$  of vertices is considered to represent arguments and the set $\att$ of edges is considered to represent the attack relation between arguments, i.e.\ the relation between a counterargument and the argument that it counters.
\end{definition}

\begin{definition}
 An $\att$-path is a sequence $\langle a_0, \dots, a_n \rangle$ of arguments where \linebreak $(a_i,a_{i+1}) \in \att$ for $0 \leq i < n$ and where $a_j \neq a_k$ for $0 \leq j < k \leq n$ with either $j \neq 0$ or $k \neq n$. An \emph{odd $\att$-cycle} is an $\att$-path $\langle a_0, \dots, a_n \rangle$ where $a_0 = a_n$ and $n$ is odd. 
\end{definition}

\begin{definition}
Let $F = \langle \Ar,\att \rangle$ be an AF, and let $S \subseteq \Ar$. We write $F|_S$ for the restricted AF $\langle S, \att \cap (S \times S) \rangle$. The set $S$ is called \emph{conflict-free} iff there are no arguments $b,c \in S$ such that $b$ attacks $c$ (i.e. such that $(b,c) \in \att$). 
Argument $a \in \Ar$ is \emph{defended} by $S$ iff for every $b \in \Ar$ such that $b$ attacks $a$ there exists $c \in S$ such that $c$ attacks $b$.
We say that $S$ \emph{attacks} $a$ if there exists $b \in S$ such that $b$ attacks $a$, and
we define $S^+ = \{a \in \Ar \mid S \mbox{ attacks } a \}$ and $S^- = \{a \in \Ar \mid a \mbox{ attacks some } b \in S \}$.
\begin{itemize}
\item $S$ is a \emph{complete extension} of $F$ iff it is conflict-free, it defends all its arguments and it contains all the arguments it defends. 
\item $S$ is a \emph{stable extension} of $F$ iff it is conflict-free and it attacks all the arguments of $\Ar \setminus S$.
\item $S$ is the \emph{grounded extension} of $F$ iff it is a minimal with respect to set inclusion complete extension of $F$. 
\item $S$ is a \emph{preferred extension} of $F$ iff it is a maximal with respect to set inclusion complete extension of $F$.
\item $S$ is a semi-stable extension of $F$ iff it is a complete extension and 
there exists no complete extension $S_1$ such that $S \cup S^+ \subset S_1 \cup S_{1}^{+}$. 
\item $S$ is a stage extension of $F$ iff $S$ is a conflict-free set and 
there exists no conflict-free set $S_1$ such that $S \cup S^+ \subset S_1 \cup S_{1}^{+}$. 
\item $S$ is a naive extension of $F$ iff $S$ is a maximal conflict-free set. 
\end{itemize}
\end{definition}

CF2 semantics was first introduced by Baroni~\textit{et~al.}~\cite{baroni2005scc}. The idea behind it is that we partition the AF into \emph{strongly connected components} and recursively evaluate it component by component by choosing maximal conflict-free sets in each component and removing arguments attacked by chosen arguments. We formally define it following the notation of Dvo\v{r}\'ak~and~Gaggl~\cite{Dvorak16}. For this we first need some auxiliary notions:

\begin{definition}
Let $F = \langle \Ar,\att \rangle$ be an AF, and let $a,b \in \Ar$. We define $a \sim b$ iff either $a=b$ or there is an \att-path from $a$ to $b$ and there is an \att-path from $b$ to $a$. The equivalence classes under the equivalence relation $\sim$ are called \emph{strongly connected components} (SCCs) of $F$. We denote the set of SCCs of $F$ by $\SCCs(F)$. Given $S \subseteq \Ar$, we define $D_F(S) := \{b \in \Ar \mid \exists a \in S: (a,b) \in \att \land a \not \sim b \}$. 
\end{definition}

The simplified SCC-recursive scheme used for defining CF2 and stage2 is a function that maps a semantics $\sigma$ to another semantics $\scc(\sigma)$:

\begin{definition}
\label{def:SCC}
 Let $\sigma$ be an argumentation semantics. The argumentation semantics $\scc(\sigma)$ is defined as follows. Let $F = \langle \Ar,\att \rangle$ be an AF, and let $S \subseteq \Ar$. Then $S$ is an $\scc(\sigma)$-extension of $F$ iff either
 \begin{itemize}
\setlength\itemsep{0pt}
\setlength{\parskip}{0pt}
\vspace{-1.5mm}
  \item $|\SCCs(F)| = 1$ and $S$ is a $\sigma$-extension of $F$, or
  \item $|\SCCs(F)| > 1$ and for each $C \in \SCCs(F)$, $S \cap C$ is an $\scc(\sigma)$-extension of $F|_{C \setminus D_F(S)}$.
 \end{itemize}
\end{definition}

\emph{CF2 semantics} is defined to be $\scc(\textit{naive})$, and \emph{stage2} semantics is defined to be $\scc(\textit{stage})$.

%
%
%


Apart from the function $\scc$, we introduce a further function -- called $\nsa$ -- that also maps a semantics to another semantics. Informally, the idea behind $\nsa(\sigma)$ is that we first delete all self-attacking arguments and then apply $\sigma$. For defining $\nsa$ formally, we first need an auxiliary definition:

\begin{definition}
 Let $F = \langle \Ar,\att \rangle$ be an AF. We define the \emph{non-self-attacking restriction} of $F$, denoted by $\NSA(F)$, to be the AF $F|_{\Ar'}$, where $\Ar' := \{ a \in \Ar \mid (a,a) \notin \att \}$. 
\end{definition}

\begin{definition}
 Let $\sigma$ be an argumentation semantics. The argumentation semantics $\nsa(\sigma)$ is defined as follows. Let $F = \langle \Ar,\att \rangle$ be an AF, and let $S \subseteq \Ar$. We say that $E$ is an $\nsa(\sigma)$-extension of $F$ iff $E$ is a $\sigma$-extension of $\NSA(F)$.
\end{definition}

We now define the Directionality principle introduced by Baroni and Giacomin~\cite{baroni2007principle}. For this, we first need an auxiliary notion:

\begin{definition}
 Let $F = \langle \Ar,\att \rangle$ be an AF. A set $U \subseteq \Ar$ is \emph{unattacked} iff there exists no $a \in \Ar \setminus U$
such that $a$ attacks some $b \in U$. 
\end{definition}

\begin{definition}
 A semantics $\sigma$ satisfies the \emph{Directionality} principle iff for every AF $F$ and every unattacked set $U$, it holds that $\sigma({F|}_{U}) = \{ E \cap U \mid E \in \sigma(F) \}$.
\end{definition}

The Directionality principle corresponds to an important feature of the human practice of argumentation, namely that if a person has formed an opinion on some arguments and is confronted with new arguments, 
they will only feel compelled to reconsider their judgment on the prior arguments if one of the new arguments attacks one of the prior arguments. 
Apart from our own intuition, we can also refer to the results of an empirical cognitive study on argumentation that shows that humans are able to systematically judge the directionality of attacks between arguments \cite{cramer2018directionality}. 
Thus we consider the Directionality principle crucial for the goal that we focus on in this paper.

\section{Two New Principles}
\label{sec:principles}
The first new principle we consider is \emph{Irrelevance of Necessarily Rejected Arguments} (\emph{INRA}). Informally, INRA says that if an argument is attacked by every extension of an AF, then deleting this argument should not change the set of extensions. The idea here is that an argument that is attacked by every extension would not be held by any party, and hence would never be brought forwards in a debate. Hence, it should be treated as if it did not even exist.

In order to formally define the INRA principle, we first need to define a notation for an AF with one argument deleted:

\begin{definition}
 Let $F = \langle \Ar,\att \rangle$ be an AF and let $a \in \Ar$ be an argument. Then $F_{-a}$ denotes the restricted AF $F|_{\Ar \setminus \{a\}}$.
\end{definition}

\begin{definition}
 Let $\sigma$ be an argumentation semantics. We say that $\sigma$ satisfies \emph{Irrelevance of Necessarily Rejected Arguments} (\emph{INRA}) iff for every AF $F = \langle \Ar,\att \rangle$ and every argument $a \in \Ar$, if every $E \in \sigma(F)$ attacks $a$, then $\sigma(F) = \sigma(F_{-a})$.
\end{definition}

The second principle that we consider is \emph{Strong Completeness Outside Odd Cycles} (\emph{SCOOC}). Informally, SCOOC says that if an argument $a$ and its attackers are not in an odd cycle, then an extension not containing any of $a$'s attackers must contain $a$. 


In order to formally define the Strong Completeness Outside Odd Cycles principle, we first need to define the auxiliary notion of a set of arguments being \emph{strongly complete outside odd cycles}.


\begin{definition}
 Let $F = \langle \Ar,\att \rangle$ be an AF, and let $A \subseteq \Ar$. We say that $A$ is \emph{strongly complete outside odd cycles} iff for every argument $a \in \Ar$, if no argument in $\{a\} \cup \{a\}^-$ is in an odd $\att$-cycle and $A \cap \{a\}^- = \emptyset$, then $a \in A$.
\end{definition}

\begin{definition}
 Let $\sigma$ be an argumentation semantics. We say that $\sigma$ satisfies \emph{Strong Completeness Outside Odd Cycles (SCOOC)} iff for any AF $F$, every $\sigma$-extension of $F$ is strongly complete outside odd cycles.
\end{definition}

The SCOOC principle is related to the property of \emph{strong completeness}: An extension $E$ is \emph{strongly complete} iff every argument not attacked by $E$ is in $E$. We call this property \emph{strong completeness} as it is a strengthening of completeness, which states that every argument defended by $E$ is in $E$. 

The stable semantics is the only widely studied argumentation semantics that satisfies strong completeness. More precisely, the stable semantics can be characterized by the conjunction of conflict-freeness and strong completeness. In other words, one can say that the stable semantics is motivated by the idea that a violation of strong completeness constitutes a paradox and should therefore be avoided.

The stable semantics satisfies strong completeness at the price of allowing for situations in which there are no extensions and hence no judgment can be made on any argument whatsoever. Such cases are always due to odd $\att$-cycles. So we can say that odd $\att$-cycles -- unless resolved through arguments attacking the odd cycle -- cause paradoxical situations. The idea of most semantics other than stable semantics is to somehow contain these paradoxes so that they do not affect our ability to make judgments about completely or sufficiently unrelated arguments. 

The idea of the SCOOC principle is that while in odd cycles we may not be able to avoid a paradoxical judgments about the arguments, i.e.\ a judgment in which an argument is not accepted even though none of its attackers is accepted, such paradoxical judgments should be completely avoided outside of odd cycles.

How does that differ from the containment of paradoxical situations provided by existing semantics? Admissibility-based semantics do not allow for any judgment about an argument in an unattacked odd cycle; however this undecided status is not limited to odd cycles, but carries forward to arguments that are not in an odd cycle but that are $\att$-reachable from an odd cycle. 

Naive-based semantics like CF2, stage and stage2 allow for judgments about arguments in an unattacked odd cycle, but also at the cost of affecting the way arguments that are not in odd cycles are interpreted. For example, CF2 allows for a six-cycle to be interpreted in a doubly paradoxical way despite the fact that it is an even cycle that can be interpreted in a non-paradoxical manner\onlyreport{ (see Figure \ref{fig:6cycle} below)}. This behavior of CF2 was also considered problematic by  Dvo\v{r}\'ak~and~Gaggl~\cite{Dvorak16}, who used this example to motivate their stage2 semantics, but as \onlypaper{established by Theorem \ref{thm:violateSCOOC} below}\onlyreport{we will show in Figure \ref{fig:stage2ex} below}, stage2 also fails to avoid paradoxical judgments about arguments that are not themselves involved in an odd cycle. 

The SCOOC principle was designed to systematically identify whether a semantics suffers from this problem. As it turns out, all the standard semantics other than stable do suffer from the problem, i.e.\ do not satisfy SCOOC.

We will now look at which semantics satisfy or do not satisfy each of the two principles that we have defined. 

%
%

\begin{theorem}
\label{thm:IUA}
 The grounded, complete, naive and $\nsa(\CFtwo)$ semantics satisfy INRA.
\end{theorem}


\onlyreport{
Before we can prove this theorem, we first need some auxiliary definitions and lemmas.

\begin{definition}
 A semantics $\sigma$ is called \emph{SCC-rich} iff for every AF $F = \langle \Ar,\att \rangle$ such that $|\SCCs(F)| = 1$ and every argument $a \in \Ar$, there is an extension $E \in \sigma(F)$ such that $E$ does not attack $a$. 
\end{definition} 
 
\begin{definition}
 A semantics is called \emph{semi-rich} iff for every AF $F = \langle \Ar,\att \rangle$ and every argument $a \in \Ar$ such that $(a,a) \notin \att$, there is an extension $E \in \sigma(F)$ such that $E$ does not attack $a$.
\end{definition}
 
\begin{definition}
 A semantics is called \emph{SCC-semi-rich} iff for every AF $F = \langle \Ar,\att \rangle$ such that $|\SCCs(F)| = 1$ and every argument $a \in \Ar$ such that $(a,a) \notin \att$, there is an extension $E \in \sigma(F)$ such that $E$ does not attack $a$.
\end{definition}


\begin{lemma}
 \label{lem:naive-rich}
 Naive semantics is semi-rich and thus also SCC-semi-rich. 
\end{lemma}
\begin{proof}
 Let $F = \langle \Ar,\att \rangle$ be an AF and let $a \in \Ar$ be an argument such that $(a,a) \notin \att$. Then $\{a\}$ is a conflict-free set, so there is some maximal conflict-free set $E \supseteq \{a\}$. Then $E$ is a naive extenstion of $F$ that does not attack $a$.\qed
\end{proof}


\begin{lemma}
 \label{lem:com-gr}
 Grounded and complete semantics are SCC-rich.
\end{lemma}
\begin{proof}
 Let $F = \langle \Ar,\att \rangle$ be an AF such that $|\SCCs(F)| = 1$ and let $a \in \Ar$. We distinguish two cases:
 \begin{enumerate}[(a)]
  \item $\att = \emptyset$. In this case $\Ar$ is the only grounded and complete extension of $F$, and $\Ar$ does not attack $a$.
  \item $\att \neq \emptyset$. Since $|\SCCs(F)| = 1$, this implies that every argument is attacked by some argument. Thus $\emptyset$ is a grounded and complete extension of $F$. Since $\emptyset$ does not attack $a$, the required condition is satisfied.
\qed
 \end{enumerate}
\end{proof}


\begin{lemma}
 \label{lem:IUA-criterion}
 Let $\sigma$ be an SCC-rich or SCC-semi-rich semantics. 
 \begin{enumerate}[(a)]
  \item  If $\sigma$ is SCC-rich, then $\scc(\sigma)$ satisfies INRA. 
  \item If $\sigma$ is SCC-semi-rich, then $\nsa(\scc(\sigma))$ satisfies INRA.
 \end{enumerate}
\end{lemma}

\begin{proof}
Before we provide the long technical proof of this lemma, we first sketch the proof idea.

First we observe that for showing that $\nsa(\scc(\sigma))$ satisfies INRA, it is enough to consider AFs without self-attacking arguments. But in such AFs, SCC-richness and SCC-semi-richness coincide. So we can actually assume SCC-richness for both parts of the lemma.

We consider an argument $a$ that is attacked by every extension and need to show that removing that argument from the AF will not result in the emergence of new extensions or the disappearance of any previous extensions. Due to the SCC-richness of $\sigma$, $a$ cannot be in an initial SCC. Instead, $a$ must be in a position where, whatever happens in the SCCs that come before $a$, some argument attacking $a$ will be accepted. Thus the SCC-recursive scheme removes $a$ from the computation of the semantics at that step. Since that is the case, removing $a$ from the AF will make no difference, because what happens in the SCCs that preceed $a$ will not be affected by the initial removal of $a$, and starting at the SCC that (originally) contains $a$, it makes no difference whether $a$ is initially removed from the framework or removed from the computation by the SCC-recursive scheme due to having an attacker from a previous SCC. 

The main difficulty in making this proof sketch a rigorous proof is that the removal of $a$ may change the structure of the SCCs, as the SCC containing $a$ may be split up into multiple SCCs. That complicates the argument significantly. We now show how this case can be covered in the full proof that follows.

We first prove part (a) of this lemma and then show how the proof can be adapted to prove part (B).

 Let $\sigma$ be an SCC-rich semantics. Let $F = \langle \Ar,\att \rangle$ be an AF and let $a \in \Ar$ be an argument such that for every $E \in \scc(\sigma)(F)$, $E$ attacks $a$. We need to show that $\scc(\sigma)(F) = \scc(\sigma)(F_{-a})$. 
 
 Note that $F$ must have more than one SCC, because otherwise the SCC-richness of $\sigma$ would imply that there is an extension that does not attack $a$. We now distinguish two cases:
 
 \underline{Case (i): $|SCCs(F_{-a})| = 1$.} In this case, $F$ has two SCCs, namely $\{a\}$ and $\Ar \setminus \{a\}$. Since $a$ is attacked by every extension, it must be attacked from some argument in $\Ar \setminus \{a\}$, so $a$ does not attack any argument in $\Ar \setminus \{a\}$. So $\Ar \setminus \{a\}$ is the initial SCC of $F$ and every extension of $F_{-a}$ contains an argument attacking $a$. So $\scc(\sigma)(F_{-a}) = \sigma(F_{-a}) = \scc(\sigma)(F)$.
  
 \underline{Case (ii): $|SCCs(F_{-a})| > 1$.} In this case, we prove the result by an induction over the number of arguments in $F$, so we may assume as induction hypothesis that it holds for strict subframeworks of $F$. Let $C_a$ denote the SCC of $F$ that contains $a$.  
 
 First we show that $\scc(\sigma)(F) \subseteq \scc(\sigma)(F_{-a})$. Let $S \in \scc(\sigma)(F)$. Let $C \in \SCCs(F)$. Since $|SCCs(F)| > 1$, it follows by Definition~\ref{def:SCC} that $S \cap C$ is an $\scc(\sigma)$-extension of $F|_{C \setminus D_{F}(S)}$. By the induction hypothesis, $S \cap C$ is an $\scc(\sigma)$-extension of $F|_{C \setminus (D_{F}(S) \cup \{a\})}$. So we have established that for each $C \in \SCCs(F)$, $S \cap C$ is an $\scc(\sigma)$-extension of $F|_{C \setminus (D_{F}(S) \cup \{a\})}$ (1). 
 
 We need to show that $S \in \scc(\sigma)(F_{-a})$. Let $C' \in \SCCs(F_{-a})$. By Definition \ref{def:SCC} and the fact that $|SCCs(F_{-a})| > 1$, it is enough to show that this arbitrarily chosen $C' \in \SCCs(F_{-a})$ satisfies the following property:
 
       \hspace{6mm} (*) \hspace{2mm} $S \cap C'$ is an $\scc(\sigma)$-extension of $F_{-a}|_{C' \setminus D_{F_{-a}}(S)}$. 
 
 \noindent Note that either $C' \in \SCCs(F)$ or that $C' \subseteq C_a$. We consider these two cases separately.
 
 If $C' \in \SCCs(F)$, then $F_{-a}|_C \setminus D_{F_{-a}}(S) = F|_{C \setminus (D_{F}(S) \cup \{a\})}$, so the required property (*) directly follows from (1).
 
 If $C' \subseteq C_a$, we show (*) by making the case distinction from the definition of $\scc(\sigma)$:
 
 \begin{enumerate}[\textnormal{Case} 1:]
  \item $|\SCCs(F|_{C_a} \setminus (D_{F}(S) \cup \{a\}))| =1$. Then $|\SCCs(F_{-a}|_{C_a} )| = 1$, so the single SCC of $F|_{C_a} \setminus (D_{F}(S) \cup \{a\})$ must either be fully contained in $C'$ or disjoint from $C'$. In the first case, $C' = C_a \setminus \{a\}$, so $S \cap C' = S \cap C_a$ $F|_{C_a} \setminus (D_{F}(S) \cup \{a\}) = F_{-a}|_{C'} \setminus D_{F_{-a}}(S)$. Therefore (1) applied to $C_a$ implies that property (*) holds. In the second case, $C' \subseteq D_{F_{-a}}(S)$ and $S \cap C'= \emptyset$, so (*) holds because the empty set is a $\scc(\sigma)$-extension of the empty framework.
  \item $|\SCCs(F|_{C_a} \setminus (D_{F}(S) \cup \{a\}))| > 1$. Then by Definition~\ref{def:SCC}, for each $C^* \in \SCCs(F|_{C_a} \setminus (D_{F}(S) \cup \{a\}))$, $S \cap C^*$ is an $\scc(\sigma)$-extension of $F|_{C^*} \setminus D_{F|_{C_a} \setminus (D_{F}(S) \cup \{a\})}(S)$. Let $C'' \in \SCCs(F|_{C'} \setminus D_{F}(S))$ (2). Then $C'' \in \SCCs(F|_{C_a} \setminus (D_{F}(S) \cup \{a\}))$, because $F|_{C'} \setminus D_{F}(S) \subseteq F|_{C_a} \setminus (D_{F}(S) \cup \{a\})$ (note that $C''$ cannot be expanded to a larger SCC in $F|_{C_a} \setminus (D_{F}(S) \cup \{a\})$, because $C'$ is an SCC of $F$ and would therefore have to contain this expansion of $C''$). Now (2) together with Definition~\ref{def:SCC} implies property (*). 
 \end{enumerate}
 This concludes the proof that $S \in \scc(\sigma)(F_{-a})$ and thus that $\scc(\sigma)(F) \subseteq \scc(\sigma)(F_{-a})$
 The proof that $\scc(\sigma)(F_{-a}) \subseteq \scc(\sigma)(F)$ works similarly.

 The main part of the proof of part (b) works similarly, but the beginning is a bit different:

 Let $\sigma$ be an SCC-semi-rich semantics. Let $F = \langle \Ar,\att \rangle$ be an AF and let $a \in \Ar$ be an argument such that for every $E \in \nsa(\scc(\sigma))(F)$, $a \notin E$. We need to show that $\nsa(\scc(\sigma))(F) = \nsa(\scc(\sigma))(F_{-a})$. Define $F' := \NSA(F)$. We need to show that $\scc(\sigma)(F) = \scc(\sigma)(F_{-a})$. If $a$ is not an argument in $F'$, then the result trivially holds. So suppose $a$ is in $F'$. Note that $F'$ must have more than one SCC, because otherwise the SCC-semi-richness of $\sigma$ would imply that there is an extension that does not attack $a$. 
 
 Now we continue as in the proof of part (a), just with $F'$ in place of $F$.
 \qed
\end{proof}

\renewcommand*{\proofname}{Proof of Theorem \ref{thm:IUA}}

\begin{proof}
 By Lemmas \ref{lem:naive-rich}, \ref{lem:com-gr} and \ref{lem:IUA-criterion} and the fact that $\grounded = \scc(\grounded)$, $\complete = \scc(\complete)$ and ${\nsa(\CFtwo)} = \nsa(\scc(\naive))$, it directly follows that grounded, complete and $\nsa(\CFtwo)$ satisfy INRA.
 
 We now show that naive semantics satisfies INRA. Let $F = \langle \Ar,\att \rangle$ be an AF and let $a \in \Ar$ be an argument such that for every $E \in \naive(F)$, $E$ attacks $a$. By the semi-richness of the naive semantics (Lemma \ref{lem:naive-rich}), it follows that $(a,a) \in \att$. 
 
 We need to show that $\naive(F) = \naive(F_{-a})$. Let $S \in \naive(F)$. As $a \notin S$, $S \subseteq \Ar \setminus \{a\}$. $S$ is conflict-free, and as $S$ is maximal with this property in $F$, it is also maximal with this property in $F_{-a}$. So $S \in \naive(F_{-a})$, as required.
 
 Now let $S \in \naive(F_{-a})$. $S$ is conflict-free. Since $(a,a) \in \att$, $S \cup \{a\}$ is not conflict-free. Together with the maximality of $S$ in $F_{-a}$, this implies that $S$ is a maximally conflict-free subset of $\Ar$, i.e.\ $S \in \naive(F)$, as required. \qed
\end{proof}

\renewcommand*{\proofname}{Proof}
}

\begin{theorem}
\label{thm:violateINRA}
 Stable, preferred, semi-stable, stage, stage2 and CF2 semantics violate INRA.
\end{theorem}

\onlyreport{
\begin{proof}
A counterexample for stable, preferred, semi-stable, stage and stage2 semantics is in Figure~\ref{fig:counterex1}. A counterexample for CF2 semantics is in Figure~\ref{fig:counterexCF2}. \qed
\end{proof}

\begin{figure}[!h]
\vspace{-6mm}
\begin{center}
\begin{tikzpicture}[->,>=stealth,shorten >=1pt,auto,node distance=1.5cm,
  thick,main node/.style={circle,draw,font=\small}]
  \node[main node] (a) at (0,0) {$a$};
  \node[main node] (b) at (1.6,0) {$b$};
  \node[main node] (c) at (0.8,0.9) {$c$};
  \draw[->,>=latex] (a) to (b);
  \draw[->,>=latex] (b) to (c);
  \draw[<->,>=latex] (c) to (a);
	\end{tikzpicture} \\ 
\end{center}
\vspace{-5mm}
\caption{Stable, preferred, semi-stable, stage and stage2 semantics violate INRA, since the only extension $\{a\}$ attacks $b$, but removing $b$ yields an additional extension, namely $\{c\}$.}
\label{fig:counterex1}
\end{figure}
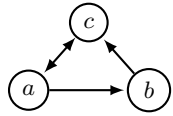

\begin{figure}[!h]
\vspace{-11mm}
\begin{center}
\begin{tikzpicture}[->,>=stealth,shorten >=1pt,auto,node distance=1.5cm,
  thick,main node/.style={circle,draw,font=\small}]
  \node[main node] (a) at (0,0) {$a$};
  \node[main node] (b) at (1.6,0) {$b$};
  \node[main node] (c) at (0.8,0.9) {$c$};
  \draw[->,>=latex] (a) to (b);
  \draw[->,>=latex] (b) to (c);
  \draw[<->,>=latex] (c) to (a);
   \path[->,every loop/.style={looseness=10}] (c) edge  [in=120,out=60,loop] node {} ();  
	\end{tikzpicture} \\ 
\end{center}
\vspace{-5mm}
\caption{CF2 semantics violates INRA, since both extension ($\{a\}$ and $\{b\}$) attack $c$, but after removing $c$, $\{b\}$ is no longer an extension.}
\label{fig:counterexCF2}
\vspace{-3mm}
\end{figure}
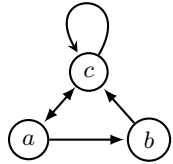
}

\begin{theorem}
 Stable semantics satisfies SCOOC.
\end{theorem}

\onlyreport{
\begin{proof}
Consider an AF $F$, a stable extension $E$ of $F$ and an argument $a \in \Ar$, such that $E \cap \{a\}^- = \emptyset$. Then by definition of stable semantics we have $a \in E$. Consequently, $E$ is strongly complete, and in particular $E$ is strongly complete outside odd cycles. 
\qed
\end{proof}
}

\begin{theorem}
\label{thm:violateSCOOC}
Complete, grounded, preferred, semi-stable, naive, stage, CF2, stage2 and $\nsa(\CFtwo)$ semantics violate SCOOC.
\end{theorem}

\onlyreport{
\begin{proof}
A counterexample of complete, grounded, preferred and semi-stable is in Figure~\ref{fig:selfattack}, a counterexample for naive, CF2 and $\nsa(\CFtwo)$ is in Figure~\ref{fig:6cycle}, a counterexample for naive and stage is in Figure~\ref{fig:selfattack2}, and a counterexample for stage2 is in Figure~\ref{fig:stage2ex}.
\qed
\end{proof}

\vspace{-9mm}

\begin{figure}[!h]
\begin{center}
\begin{tikzpicture}[->,>=stealth,shorten >=1pt,auto,node distance=1.5cm,
  thick,main node/.style={circle,draw,font=\small}]
  \node[main node] (1) {$a$};
	\node[main node] (2) [right of=1] {$b$};
	\node[main node] (3) [right of=2] {$c$};
  \draw[->,>=latex] (1) to (2);
  \draw[->,>=latex] (2) to (3);
   \path[->,every loop/.style={looseness=10}] (1) edge  [in=210,out=150,loop] node {} ();  
	\end{tikzpicture} \\ 
\end{center}
\vspace{-7mm}
\caption{Complete, grounded, preferred and semi-stable semantics violate SCOOC, since $E=\{\}$ is an extension but $E$ is not strongly complete outside odd cycles: $b$ and $c$ are not in an odd cycle, $\{c\}^-=\{b\}$, but $E$ does not contain $c$.}
\label{fig:selfattack}
\end{figure}
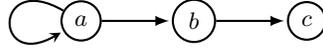

\vspace{-7mm}

\begin{figure}[!h]
\begin{center}
\begin{tikzpicture}[->,>=stealth,shorten >=1pt,auto,node distance=1.5cm,
  thick,main node/.style={circle,draw,font=\small}]
  \node[main node] (1) {$a$};
	\node[main node] (2) [right of=1] {$b$};
	\node[main node] (3) [right of=2] {$c$};
	\node[main node] (4) [below of=3] {$d$};
	\node[main node] (5) [below of=2] {$e$};
	\node[main node] (6) [below of=1] {$f$};
  \draw[->,>=latex] (1) to (2);
  \draw[->,>=latex] (2) to (3);
  \draw[->,>=latex] (3) to (4);
  \draw[->,>=latex] (4) to (5);
  \draw[->,>=latex] (5) to (6);
  \draw[->,>=latex] (6) to (1);
  \end{tikzpicture} \\ 
\end{center}
\vspace{-4mm}
\caption{Naive, CF2 and $\nsa(\CFtwo)$ semantics violate SCOOC, since $E=\{a,d\}$ is an extension but $E$ is not strongly complete outside odd cycles: $b$ and $c$ are not in an odd cycle, $\{c\}^-=\{b\}$, but $E$ does not contain $c$.}
\label{fig:6cycle}
\end{figure}

\vspace{-8mm}

\begin{figure}[!h]
\begin{center}
\begin{tikzpicture}[->,>=stealth,shorten >=1pt,auto,node distance=1.5cm,
  thick,main node/.style={circle,draw,font=\small}]
  \node[main node] (1) {$a$};
	\node[main node] (2) [right of=1] {$b$};
	\node[main node] (3) [right of=2] {$c$};
  \draw[->,>=latex] (1) to (2);
  \draw[->,>=latex] (2) to (3);
   \path[->,every loop/.style={looseness=10}] (3) edge  [in=30,out=330,loop] node {} ();  
	\end{tikzpicture} \\ 
\end{center}
\vspace{-7mm}
\caption{Stage and naive semantics violate SCOOC, since $E=\{b\}$ is an extension but $E$ is not strongly complete outside odd cycles: $a$ is not in an odd cycle, $\{a\}^-=\{\}$, but $E$ does not contain $a$.}
\label{fig:selfattack2}
\vspace{2mm}
\end{figure}
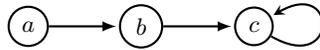

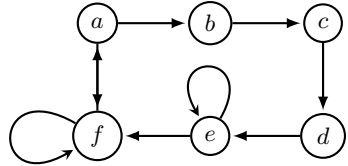
\begin{figure}[!h]
\begin{center}
\begin{tikzpicture}[->,>=stealth,shorten >=1pt,auto,node distance=1.5cm,
  thick,main node/.style={circle,draw,font=\small}]
  \node[main node] (1) {$a$};
	\node[main node] (2) [right of=1] {$b$};
	\node[main node] (3) [right of=2] {$c$};
	\node[main node] (4) [below of=3] {$d$};
	\node[main node] (5) [below of=2] {$e$};
	\node[main node] (6) [below of=1] {$f$};
  \draw[->,>=latex] (1) to (2);
  \draw[->,>=latex] (2) to (3);
  \draw[->,>=latex] (3) to (4);
  \draw[->,>=latex] (4) to (5);
  \draw[->,>=latex] (5) to (6);
  \draw[->,>=latex] (6) to (1);
  \draw[->,>=latex] (1) to (6);
     \path[->,every loop/.style={looseness=10}] (5) edge  [in=120,out=60,loop] node {} ();  
   \path[->,every loop/.style={looseness=10}] (6) edge  [in=210,out=150,loop] node {} ();  

  \end{tikzpicture} \\ 
\end{center}
\vspace{-9mm}
\caption{Stage2 semantics violates SCOOC, since $E=\{a,d\}$ is an extension but $E$ is not strongly complete outside odd cycles: $b$ and $c$ are not in an odd cycle, $\{c\}^-=\{b\}$, but $E$ does not contain $c$.}
\label{fig:stage2ex}
\end{figure}
}

\section{SCF2 Semantics}
\label{sec:SCF2}
In this section, we define and study the new semantics SCF2, which satisfies both of the new principles introduced in the previous section as well as the Directionality principle defined in the preliminaries. Furthermore, we will motivate the design choices in the definition of SCF2 by looking at how semantics defined in a similar way as SCF2 fail to satisfy at least one of Directionality, INRA or SCOOC.

We have seen in the previous section that $\nsa(\CFtwo)$ satisfies INRA but does not satisfy SCOOC. The idea behind the definition of SCF2 is that we modify the definition of $\nsa(\CFtwo)$ by already enforcing SCOOC at the level of the single SCCs considered in the SCC-recursive definition of $\nsa(\CFtwo)$. For this, we define a variant of naive semantics called \emph{SCOOC-naive semantics}.

\begin{definition}
  Let $F = \langle \Ar,\att \rangle$ be an AF, and let $A \subseteq \Ar$. We say that $A$ is an \emph{SCOOC-naive extension} of $F$ if $A$ is subset-maximal among the conflict-free subsets of $\Ar$ that are strongly complete outside odd cycles.
\end{definition}

Recall that CF2 is defined to be $\scc(\naive)$, i.e.\ $\nsa(\CFtwo) = \nsa(\scc(\naive))$. For defining SCF2, we just replace naive semantics by SCOOC-naive semantics in this definition.

\begin{definition}
SCF2 semantics is defined to be $\nsa(\scc(\SCOOCnaive))$.
\end{definition}

In other words, SCF2 works by first deleting all self-attacking arguments and then applying the SCC-recursive scheme that is also used in the definition of CF2, but applying SCOOC-naive semantics instead of naive semantics to each single SCC. \onlyreport{

As we will show below,} SCF2 satisfies Directionality, INRA and SCOOC, which we have argued to be desirable principles when evaluating a semantics designed to correspond well to what humans would consider a rational judgment on the acceptability of arguments. The somewhat complex definition of SCF2 raises the question whether a simpler definition could also be enough to satisfy these three principles. 

To approach this question systematically, we would like to point out that the definition of SCF2 contains three features that distinguishes it from naive semantics: It starts by deleting all self-attacking arguments (the function $\nsa$), it proceeds by applying the SCC-recursive scheme (the function $\scc$), and within each SCC, it applies SCOOC-naive rather than naive semantics. If we consider each of these three features a switch that we can switch on or off, we have eight definitions of semantics, namely $\naive$, $\nsa(\naive)$, $\SCOOCnaive$, $\nsa(\SCOOCnaive)$, $\scc(\naive)$, $\nsa(\scc(\naive))$, $\scc(\SCOOCnaive)$ and $\nsa(\scc(\SCOOCnaive))$. One can easily see that $\naive = \nsa(\naive)$, so these eight definitions define only seven different semantics, whose properties we now study in order to show that only SCF2 satisfies all three principles Directionality, INRA and SCOOC. 


Table 1 shows which of these seven semantics satisfies which of these three principles (we use the standard name CF2 for $\scc(\naive)$ and use the short name SCF2 to refer to $\nsa(\scc(\SCOOCnaive))$). Note that SCF2 satisfies all three principles, while no other of these seven semantics satisfies all three principles.


\begin{table}
\begin{center}
\begin{tabular}{|l||c|c|c|}
\rule{1.54cm}{0pt}&\rule{1.54cm}{0pt}&\rule{1.54cm}{0pt}&\rule{1.54cm}{0pt}\\[-2.8ex]
 \hline
			& \hspace{-0.4mm} Directionality \hspace{-0.4mm} & INRA 	& SCOOC \\
 \hline
 \hline
$\naive = \nsa(\naive)$ & $\times$	& $\checkmark$	& $\times$ \\
 \hline
SCOOC-naive 		& $\times$	& $\times$	& $\checkmark$ \\
 \hline
$\nsa(\SCOOCnaive)$ 	& $\times$	& $\times$	& $\checkmark$ \\
 \hline
CF2			& $\checkmark$	& $\times$	& $\times$ \\
 \hline
$\nsa(\CFtwo)$		& $\checkmark$	& $\checkmark$	& $\times$ \\
 \hline
$\scc(\SCOOCnaive)$	& $\checkmark$	& $\times$	& $\checkmark$ \\
 \hline
SCF2			& $\checkmark$	& $\checkmark$	& $\checkmark$ \\
 \hline
\end{tabular}
\label{table}
\vspace{1mm}
\caption{Properties of SCF2 and six semantics that are related to it with respect the three principles considered in this paper}
\end{center}
\vspace{-7mm}
\end{table}


\onlypaper{The results displayed in Table 1 are proven in a technical report \cite{SCF2technicalreport}. Additionally, we prove there that every AF has an SCF2 extension.}

\onlyreport{For the rest of this section, we will prove the results depicted in Table 1 as well as the theorem that every AF has an SCF2 extension. In order to prove properties about SCF2, we first need another lemma.

\begin{lemma}
\label{lem:SCOOCrich}
 SCOOC-naive semantics is SCC-semi-rich.
\end{lemma}

\begin{proof}
 Let $F = \langle \Ar,\att \rangle$ be an AF such that $|\SCCs(F) = 1|$ and let $c \in \Ar$ be an argument such that $(c,c) \notin \att$. We prove the lemma by induction over $|\Ar|$, so we assume for the inductive hypothesis that it holds for strict subframeworks of $F$. We need to find an SCF2-extension $E$ of $F$ such that $E$ does not attack $c$. We do this by specifying a non-deterministic procedure to construct such an extension. The constructed extension will be constructed in such a way that it must contain $c$, which implies that it does not attack $c$.
 
 At each step $k$ of the procedure, we identify a set $E_k$ of arguments to be included in $E$ and a set $\bar E_k$ of arguments for which we rule out that they may be in $E$. We define $U_k$ to be set of arguments $a$ such that it is not yet determined by step $k$ whether $a$ is included in $E$ or not; formally, $U_k := \Ar \setminus \bigcup_{i<k} (E_k \cup \bar E_k)$. We set $\bar E_0 := \{ a \in \Ar \mid (a,a) \in \att \}$ and $E_0 := \{c\}$. Let $k > 0$. We set $\bar E_k := \{ a \in \Ar \mid \textnormal{some argument in } E_{k-1} \textnormal{ attacks } a \textnormal{ or is attacked by } a \}$. For the definition of $E_k$ we have a case-distinction:
 \begin{enumerate}
  \item If $E_{k-1} \neq \emptyset$, we define $E_k := E^1_k \cup E^2_k$, where $E^1_k := \{ a \in U_k \mid \textnormal{all}$ attakers of $a$ are in some $\bar E_i$ for $i<k \}$, and $E^2_k := \{a \in U_k \mid \textnormal{for some } b \in \bar E_{k-1}$ that is not in an odd cycle and whose attackers are not in an odd cycle and not in any $E_i$ for $i < k$, $a$ attacks $b \}$. 
  \item If $E_{k-1} = \emptyset$ and $U_k \neq \emptyset$, we set $E_k$ to be a SCOOC-naive extension of an unattacked SCC of $U_k$ (which exists and is non-empty by inductive hypothesis and by the fact that $U_k$ does not contain self-attacking arguments, as these are all in $\bar E_0$). 
  \item If $E_{k-1} = \emptyset$ and $U_k = \emptyset$, the procedure stops.
 \end{enumerate}
 When the procedure stops at step $k$, we set $E$ to be the union of all $E_i$ for $i<k$. Note that once case 2 is applied, only case 2 can be applied. Let $n$ denote the step at which case 2 is first applied. Note that a simple proof by induction establishes that 
 whenever $a \in E^2_k$ for $k < n$, then there is an even-length path from $a$ to $c$ (1). This can be generalized to the statement that whenever $a \in E_k$ for $k < n$, there are $a',k'$ such that $a' \in \bar E_{k'}$ and there is an odd-length $\att$-path from $a'$ to $c$ and from $a'$ to $a$ (2).

 $a \in E$ because $a \in E_0$. It is easy to see from the construction of $E$ that $E$ is strongly complete outside odd cycles and that adding an argument to $E$ creates a conflict within $E$. So all we still need to prove is that $E$ is conflict-free. Suppose for a contradiction that $(a_1,a_2) \in \att$ for $a_1,a_2 \in E$. From the definition of $\bar E_k$ for $k>0$ it follows that there is an $i$ such that $a_1,a_2 \in E_i$. Clearly this is not the case if $E_i$ was defined according to case 2 of the definition of $E_k$, so $k<n$. By the definition of $E^1_k$, it furthermore follows that $a_2 \notin E^1_i$. So $a_2 \in E^2_i$, i.e.\ by property (1) there is an even-length $\att$-path $p_0$ from $a_2$ to $c$. Furthermore, by property (2) there is an $a_1'$ such that there is an odd-length $\att$-path $p_1$ from $a_1'$ to $c$ and from $a_1'$ to $a_1$. Since $|\SCCs(F) = 1|$, there is a path $p_2$ from $a_2$ to $a_1'$. The length of path $p_2$ must be even, because if it were odd, then concatenating $p_2$ with the odd-length path $p_1$ and the attack from $a_1$ to $a_2$ would result in an odd cycle through $a_2$, which cannot exist as $a_2 \in E^2_i$. Since $|\SCCs(F) = 1|$, there is a path $p_3$ from $c$ to $a_2$. If the length of $p_3$ is odd, then $p_3$ concatenated with the even-length path $p_0$ is an odd cycle through $a_2$. If the length of $p_3$ is even, the $p_3$ concatenated with the even-length path $p_2$ and the attack from $a_1'$ to $c$ is an odd cycle through $a_2$. so in either case, there is an odd cycle through $a_2$, which contradicts the fact that $a_2 \in E^2_i$. This completes our proof by contradiction that $E$ is conflict-free.
 \qed
\end{proof}

\begin{theorem}
\label{thm:exSCF2}
 Every AF has at least one SCF2 extension.
\end{theorem}

\begin{proof}
 Lemma \ref{lem:SCOOCrich} implies that every single-SCC AF has a SCOOC-naive extension. This together with the definition of the SCC recursive scheme implies that every AF has at least one $\scc(\SCOOCnaive)$-extension, and hence at least one SCF2 extension. \qed
\end{proof}

We now prove the properties of SCF2 listed in Table 1.

%
%

\begin{theorem}
\label{thm:direct}
 SCF2 satisfies Directionality.
\end{theorem}

\begin{proof}

Consider an AF $F$ and an unattacked set $U$. Define $F' = (\Ar',\att')$ to be $\NSA(F)$. Note that $U$ is unattacked in $F'$ and that 
therefore $\SCCs(F') = \SCCs(F'|_U) \cup \SCCs(F'|_{\Ar' \setminus U})$ (1).

First, we show that if $E_1 \in \SCFtwo({F|}_{U})$ then there exists $E_2 \in \SCFtwo(F)$ such that $E_1=E_2\cap U$.
Assume $E_1 \in \SCFtwo({F|}_{U})$. Then $E_1 \in \scc(\SCOOCnaive)(F'|_U)$, i.e.\ for each $C \in \SCCs(F'|_{U})$, $E_1 \cap C$ is an $\scc(\SCOOCnaive)$-extension of $F'|_{C} \setminus D_{F'}(E_1)$ (2). By Theorem \ref{thm:exSCF2}, $F'_{\Ar\setminus U} \setminus D_{F'}(E_1)$ has an SCF2 extension, say $E$. Define $E_2 := E_1 \cup E$. Note that $E_1=E_2\cap U$, so it is enough to show that $E_2$ is an SCF2-extension of $F$. Since $F' = \NSA(F')$, $E$ is an $\scc(\SCOOCnaive)$-extension of $F'_{\Ar\setminus U} \setminus D_{F'}(E_1)$, i.e.\ for each $C \in \SCCs(F'|_{\Ar \setminus U})$, $E \cap C$ is an $\scc(\SCOOCnaive)$-extension of $F'|_{C} \setminus D_{F'}(E)$. This together with (1) and (2) implies that $E_2$ is an $\scc(\SCOOCnaive)$-extension of $F'$, i.e.\ that $E_2$ is an SCF2-extension of $F$, as required.

Now, we show that if $E \in \SCFtwo(F)$ then $E\cap U \in \SCFtwo({F|}_{U})$. Assume $E \in \SCFtwo(F)$.
Then for each $C \in \SCCs(F')$, $E \cap C$ is an $\scc(\SCOOCnaive)$-extension of $F'|_C \setminus D_F(A)$.
By (1), it follows that $E\cap U$ is an $\scc(\SCOOCnaive)$-extension of $F'$, i.e.\ that $E \cap U$ is an SCF2-extension of $F$, as required. \qed
\end{proof}

\begin{theorem}
\label{thm:scooc}
 SCF2 satisfies SCOOC.
\end{theorem}

\begin{proof}

Consider an AF $F$, an SCF2 extension $E$ of $F$ and an argument $a \in \Ar$ such that no argument in $\{a\} \cup a^-$ is in an odd cycle and $E \cap a^- = \emptyset$. Then by definition of SCF2 semantics, the moment the SCOOC-naive function is applied to a sub-framework of $F$ containing $a$, we have $a \in E$. Consequently, $E$ is strongly complete outside odd cycles. \qed 
\end{proof}

\begin{theorem}
 SCF2 satisfies INRA.
\end{theorem}

\begin{proof}
 By Lemma \ref{lem:SCOOCrich}, SCOOC-naive semantics is SCC-semi-rich. So by Lemma \ref{lem:IUA-criterion} and the definition of SCF2 it follows that SCF2 satisfies INRA.  \qed
\end{proof}

Next we establish the remaining positive results from Table 1.

%
%
%
%

\begin{theorem}
 CF2, $\nsa(\CFtwo)$ and $\scc(\SCOOCnaive)$ satisfy Directionality.
\end{theorem}

\begin{proof}
 The result for CF2 has been shown by Baroni and Giacomin \cite{baroni2007principle} and directly implies the result for $\nsa(\CFtwo)$. The proof for $\scc(\SCOOCnaive)$ works just like the proof that SCF2 satisfies Directionality (Theorem \ref{thm:direct}).
 \qed
\end{proof}

%

\begin{theorem}
 SCOOC-naive, $\nsa(\SCOOCnaive)$ and $\scc(\SCOOCnaive)$ satisfy SCOOC.
\end{theorem}

\begin{proof}
 For SCOOC-naive and $\nsa(\SCOOCnaive)$, this follows directly from the definitions. As for $\scc(\SCOOCnaive)$, the proof that SCF2 satisfies SCOOC (Theorem \ref{thm:scooc}) also establishes that $\scc(\SCOOCnaive)$ satisfies SCOOC.
 \qed
\end{proof}

We now prove the negative results shown in Table 1 (ommiting the ones that are already covered by Theorems \ref{thm:violateINRA} and \ref{thm:violateSCOOC}).

\begin{theorem}
 Naive semantics violates Directionality.
\end{theorem}

\begin{proof}
 A counterexample is shown in Figure~\ref{fig:simple_attack}.
\end{proof}

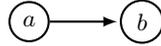
\begin{figure}[!h]
\begin{center}
\begin{tikzpicture}[->,>=stealth,shorten >=1pt,auto,node distance=1.5cm,
  thick,main node/.style={circle,draw,font=\small}]
  \node[main node] (1) {$a$};
	\node[main node] (2) [right of=1] {$b$};
  \draw[->,>=latex] (1) to (2);
	\end{tikzpicture} \\ 
\end{center}
\vspace{-7mm}
\caption{Naive semantics violates Directionality, because $a$ is not in the extension $\{b\}$, even though it is in the only extension of the unattacked subframework induced by $\{a\}$.}
\label{fig:simple_attack}
\end{figure}

\begin{theorem}
 SCOOC-naive semantics violates Directionality and INRA.
\end{theorem}

\begin{proof}
 A counterexample to both principles is shown in Figure~\ref{fig:selfattack_reinstatement}.
\end{proof}

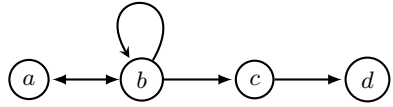
\begin{figure}[!h]
\begin{center}
\begin{tikzpicture}[->,>=stealth,shorten >=1pt,auto,node distance=1.5cm,
  thick,main node/.style={circle,draw,font=\small}]
  \node[main node] (1) {$b$};
	\node[main node] (0) [left of=1] {$a$};
	\node[main node] (2) [right of=1] {$c$};
	\node[main node] (3) [right of=2] {$d$};
  \draw[<->,>=latex] (1) to (0);
  \draw[->,>=latex] (1) to (2);
  \draw[->,>=latex] (2) to (3);
   \path[->,every loop/.style={looseness=10}] (1) edge  [in=120,out=60,loop] node {} ();  
	\end{tikzpicture} \\ 
\end{center}
\vspace{-3mm}
\caption{SCOOC-naive semantics violates Directionality, because $c$ is not in the extension $\{a,d\}$, even though it is in the only extension of the unattacked subframework induced by $\{a,b,c\}$. 
SCOOC-naive semantics violates INRA, because $b$ is attacked by both extensions ($\{a,c\}$ and $\{a,d\}$) and the extension $\{a,d\}$ is not an extension of the subframework induced by $\{a,c,d\}$.}
\label{fig:selfattack_reinstatement}
\end{figure}

\begin{theorem}
 $\nsa(\SCOOCnaive)$ semantics violates Directionality and INRA.
\end{theorem}

\begin{proof}
 A counterexample to Directionality is shown in Figure~\ref{fig:3cycle_reinstatement}. A counterexample to INRA is shown in Figure~\ref{fig:attacked_3cycle}.
\end{proof}

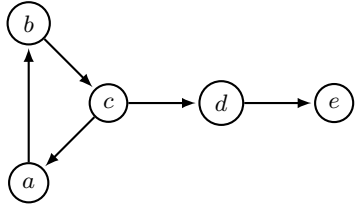
\begin{figure}[!h]
\begin{center}
\begin{tikzpicture}[->,>=stealth,shorten >=1pt,auto,node distance=1.5cm,
  thick,main node/.style={circle,draw,font=\small}]
  \node[main node] (1) {$c$};
	\node[main node] (2) [right of=1] {$d$};
	\node[main node] (3) [right of=2] {$e$};
	\node[main node] (4) [below left of=1] {$a$};
	\node[main node] (5) [above left of=1] {$b$};
  \draw[->,>=latex] (1) to (2);
  \draw[->,>=latex] (2) to (3);
  \draw[->,>=latex] (4) to (5);
  \draw[->,>=latex] (5) to (1);
  \draw[->,>=latex] (1) to (4);
	\end{tikzpicture} \\ 
\end{center}
\vspace{-7mm}
\caption{$\nsa(\SCOOCnaive)$ semantics violates Directionality, because $\{b,e\}$ is an extension, even though $\{b\}$ is not an extension of the unattacked subframework induced by $\{a,b,c,d\}$. 
}
\label{fig:3cycle_reinstatement}
\end{figure}

\begin{figure}[!h]
\begin{center}
\begin{tikzpicture}[->,>=stealth,shorten >=1pt,auto,node distance=1.5cm,
  thick,main node/.style={circle,draw,font=\small}]
  \node[main node] (1) {$a$};
	\node[main node] (2) [right of=1] {$b$};
	\node[main node] (3) [above right of=2] {$c$};
	\node[main node] (4) [below right of=2] {$d$};
  \draw[->,>=latex] (1) to (2);
  \draw[->,>=latex] (2) to (3);
  \draw[->,>=latex] (3) to (4);
  \draw[->,>=latex] (4) to (2);
	\end{tikzpicture} \\ 
\end{center}
\vspace{-7mm}
\caption{$\nsa(\SCOOCnaive)$ semantics violates INRA, because $b$ is attacked by every extension and the extension $\{a,d\}$ is not an extension of the subframework induced by $\{a,c,d\}$.}
\label{fig:attacked_3cycle}
\end{figure}
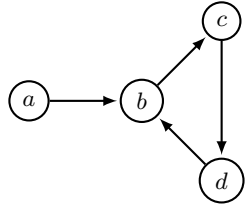

\begin{theorem}
 $\scc(\SCOOCnaive)$ semantics violates INRA.
\end{theorem}

\begin{proof}
 A counterexample is shown in Figure~\ref{fig:3cycle_with_selfattack}.
\end{proof}

\begin{figure}[!h]
\begin{center}
\begin{tikzpicture}[->,>=stealth,shorten >=1pt,auto,node distance=1.5cm,
  thick,main node/.style={circle,draw,font=\small}]
  \node[main node] (1) {$a$};
	\node[main node] (2) [above right of=1] {$b$};
	\node[main node] (3) [below right of=2] {$c$};
  \draw[<->,>=latex] (1) to (2);
  \draw[->,>=latex] (2) to (3);
  \draw[->,>=latex] (3) to (1);
   \path[->,every loop/.style={looseness=10}] (1) edge  [in=210,out=150,loop] node {} ();  
	\end{tikzpicture} \\ 
\end{center}
\vspace{-7mm}
\caption{$\scc(\SCOOCnaive)$ semantics violates INRA because $a$ is attacked by every extension and the extension $\{c\}$ is not an extension of the subframework induced by $\{b,c\}$.}
\label{fig:3cycle_with_selfattack}
\end{figure}
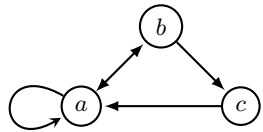
}

\section{Empirical cognitive studies}
\label{sec:empirical}
Rahwan~et~al.~\cite{rahwan2009argumentation} argue that Artificial Intelligence research will benefit from the interplay between logic and cognition and that therefore ``logicians and computer scientists ought to give serious attention to cognitive plausibility when assessing formal models of reasoning, argumentation, and decision making''. Based on the observation that in the previous literature on formal argumentation theory, an example-based approach and a principle-based approach were used to motivate and validate argumentation semantics, they propose to complement these approaches by an \emph{experiment-based approach} that takes into account empirical cognitive studies on how humans interpret and evaluate arguments. They made a first contribution to this new approach by presenting and discussing the results of two such studies that they conducted in order to test the cognitive plausibility of simple and floating reinstatement \cite{rahwan2009argumentation}.

While the argumentation frameworks used in Rahwan et al.'s studies could not distinguish between preferred semantics and naive-based semantics like CF2, two more recent studies by Cramer and Guillaume \cite{cramer2018empirical,cramer2019empirical} address this issue. Both of these studies made use of a group discussion methodology that is known to stimulate more rational thinking. According to the results of the first study \cite{cramer2018empirical}, CF2, SCF2, stage and stage2 semantics are significantly better predictors for human judgments on the acceptability of arguments than admissibility-based semantics like grounded, preferred, complete or semi-stable (binomial tests, all $p$-values $< 0.001$). However, this study did not involve argumentation frameworks that allow to distinguish between CF2, SCF2, stage and stage2 semantics. 

According to the results of Cramer and Guillaume's second study \cite{cramer2019empirical}, SCF2, CF2 and grounded semantics are better predictors for human judgments on the acceptability of arguments than stage, stage2, preferred or semi-stable semantics (binomial tests, all $p$-values $< 0.001$). Additionally, the results suggest that SCF2 is a better predictor than CF2 and grounded semantics, but the results for that are not significant. We will now explain these results in more depth.

As explained in Section \ref{sec:principles}, Dvo\v{r}\'ak and Gaggl~\cite{Dvorak16} critique a feature of CF2 semantics, namely that in the case of a six-cycle\onlyreport{ as depicted in Figure \ref{fig:6cycle}}, CF2 allows two opposite arguments \onlyreport{(e.g. $a$ and $d$)} to be accepted together. The second study by Cramer and Guillaume \cite{cramer2019empirical} confirms that this criticism is in line with human judgments of argument acceptability. We briefly summarize the data on which this judgment is made (a more detailed explanation can by found in \cite{cramer2019empirical}): Based on the overall responses of the participants in the study, Cramer and Guillaume point out that 12 of the 61 participants of their study have a high frequency of incoherent responses, so that they disconsider them from the further analysis. Among the remaining 49 participants, 22 follow a simple cognitive strategy of marking arguments as \emph{Undecided} whenever there is a reason for doubt (in line with the grounded semantics), while 27 participants do not follow this strategy. Cramer and Guillaume call these 27 participants the \emph{coherent non-grounded participants}. 

In the case of 11 out of the 12 argumentation frameworks considered in the study, the majority of these 27 coherent non-grounded participants make judgements that are in line with CF2 semantics. The only exception to this is an argumentation framework involving a six-cycle, in which only 33\% of the coherent non-grounded participants make a judgement in line with CF2 semantics, while 60\% make a judgements that is in line with SCF2, stage2, preferred and semi-stable semantics.

Dvo\v{r}\'ak and Gaggl~\cite{Dvorak16} themselves had used this criticism against CF2 to motivate their stage2 semantics, but in the study by Cramer and Guillaume \cite{cramer2019empirical}, stage2 performed significantly worse than SCF2, since all five arguments in the only AF on which stage2 and SCF2 differed were evaluated by most participants (including most coherent non-grounded participants) in line with SCF2 rather than with stage2.


In combination with the principle-based argument for SCF2 presented in the previous two sections, these preliminary findings provide additional support for our hypothesis that SCF2 corresponds well to what humans consider a rational judgment on the acceptability of arguments.

\section{Related work}
\label{sec:related}
In the previous section we have already considered related empirical work. In this section we focus on work related to the principle-based approach to abstract argumentation that we have employed in this paper.

The principle-based analysis of argumentation semantics was initiated by Baroni and Giacomin \cite{baroni2007principle} to choose among the many extension-based argumentation semantics that have been proposed in the formal argumentation literature. The handbook chapter of van der Torre and Vesic \cite{vanderTorre18} gives a classification of fifteen alternatives for argumentation semantics using  twenty-seven principles discussed in the literature on abstract argumentation. Dvo\v{r}\'ak~and~Gaggl~\cite{Dvorak16} introduce stage2 semantics by showing how it satisfies various desirable properties, similarly to how we motivate SCF2 semantics in this paper. 

Moreover, additional extension-based argumentation semantics and principles have been proposed by various authors. For example, Besnard {\em et al.} \cite{DBLP:conf/comma/BesnardDHL16} introduce a system for specifying semantics in abstract argumentation called SESAME. Moreover, many principles have been proposed for alternative semantics of argumentation frameworks such as ranking semantics \cite{amgoud2013ranking}, and for extended argumentation frameworks, for example for abstract dialectical frameworks \cite{brewka2018abstract}.

The principle of Irrelevance of Necessarily Rejected Arguments is closely related to the well-studied area of dynamics of argumentation, in which also various principles have been proposed which are closely related to INRA. 
Cayrol {\em et al.} \cite{DBLP:conf/kr/CayrolSL08} were maybe the first to study revision of frameworks using a principle-based analysis, and they have been related to notions of equivalence \cite{DBLP:journals/ai/Baumann12,DBLP:journals/ai/OikarinenW11}.
Boella {\em et al.} \cite{DBLP:conf/ecsqaru/BoellaKT09} define principles for abstracting (i.e., removing) an argument, and Rienstra {\em et al.} \cite{DBLP:conf/tafa/RienstraST15} define a variety of persistence and monotony properties for argumentation semantics. Our INRA principle is inspired by and closely related to the \emph{skeptical IO monotony principle} that they define. The difference is that their principle considers adding an attack rather than removing an argument.

\section{Conclusion and Future Work}
\label{sec:conclusion}
Motivated by empirical cognitive studies on argumentation semantics, we have introduced a new naive-based argumentation semantics called SCF2. A principle-based analysis shows that it has two distinguishing features:
\begin{enumerate}
\item If an argument is attacked by all extensions, then it can never be used in a dialogue and therefore it has no effect on the acceptance of other arguments. 
We call it {\em Irrelevance of Necessarily Rejected Arguments}. 
\item Within each extension, if none of the attackers of an argument is accepted and the argument is not involved in a paradoxical relation, then the argument is accepted. We define paradoxicality as being part of an odd cycle, and we call this principle {\em Strong Completeness Outside Odd Cycles}.
\end{enumerate}

We have argued that these features together with the findings from empirical cognitive studies make SCF2 a good candidate for an argumentation semantics that corresponds well to what humans consider a rational judgment on the acceptability of arguments.

The empirical approach to abstract argumentation theory is still a relatively new approach that needs to be developed further by modifying and improving the methodology of existing studies in the design of future studies. The current paper provides a well-motivated hypothesis that can be tested more rigorously in future empirical studies, namely the hypothesis that SCF2 predicts human judgments on the acceptability of arguments better than other abstract argumentation semantics.

On the theoretical side, more work is required to determine which other principles studied in the literature are satisfied by SCF2. Moreover, dialogue-based decision procedures must be defined, and the complexity of the various decision problems must be established. Finally, an extension towards structured argumentation should be investigated.


\bibliographystyle{myplain}
\bibliography{bib}

\begin{thebibliography}{10}

\bibitem{amgoud2013ranking}
L. Amgoud and J. Ben-Naim.
\newblock {Ranking-Based Semantics for Argumentation Frameworks}.
\newblock In W. Liu, V.~S. Subrahmanian, and J. Wijsen, editors, {\em {Scalable
  Uncertainty Management}}, pages 134--147, Berlin, Heidelberg, 2013. Springer
  Berlin Heidelberg.

\bibitem{Baroni18}
P. Baroni, M. Caminada, and M. Giacomin.
\newblock {Abstract argumentation frameworks and their semantics}.
\newblock In P. Baroni, D. Gabbay, M. Giacomin, and L. van~der Torre, editors,
  {\em {Handbook of Formal Argumentation}}, pages 159--236. College
  Publications, 2018.

\bibitem{baroni2007principle}
P. Baroni and M. Giacomin.
\newblock {On principle-based evaluation of extension-based argumentation
  semantics}.
\newblock {\em Artificial Intelligence}, 171(10):675--700, 2007.
\newblock Argumentation in Artificial Intelligence.

\bibitem{baroni2005scc}
P. Baroni, M. Giacomin, and G. Guida.
\newblock {SCC-recursiveness: a general schema for argumentation semantics}.
\newblock {\em Artificial Intelligence}, 168(1):162--210, 2005.

\bibitem{DBLP:journals/ai/Baumann12}
R. Baumann.
\newblock {Normal and strong expansion equivalence for argumentation
  frameworks}.
\newblock {\em Artif. Intell.}, 193:18--44, 2012.

\bibitem{DBLP:conf/comma/BesnardDHL16}
P. Besnard, S. Doutre, V.~H. Ho, and D. Longin.
\newblock {{SESAME} - {A} System for Specifying Semantics in Abstract
  Argumentation}.
\newblock In M. Thimm, F. Cerutti, H. Strass, and M. Vallati, editors, {\em
  {Proceedings of the First International Workshop on Systems and Algorithms
  for Formal Argumentation {(SAFA)} co-located with the 6th International
  Conference on Computational Models of Argument {(COMMA} 2016), Potsdam,
  Germany, September 13, 2016.}}, volume 1672 of {\em {{CEUR} Workshop
  Proceedings}}, pages 40--51. CEUR-WS.org, 2016.

\bibitem{DBLP:conf/ecsqaru/BoellaKT09}
G. Boella, S. Kaci, and L.~W.~N. van~der Torre.
\newblock {Dynamics in Argumentation with Single Extensions: Abstraction
  Principles and the Grounded Extension}.
\newblock In C. Sossai and G. Chemello, editors, {\em {Symbolic and
  Quantitative Approaches to Reasoning with Uncertainty, 10th European
  Conference, {ECSQARU} 2009, Verona, Italy, July 1-3, 2009. Proceedings}},
  volume 5590 of {\em {Lecture Notes in Computer Science}}, pages 107--118.
  Springer, 2009.

\bibitem{brewka2018abstract}
G. Brewka, S. Ellmauthaler, H. Strass, J. Wallner, and S. Woltran.
\newblock {\em {Abstract dialectical frameworks}}.
\newblock College Publications, International, 2018.

\bibitem{DBLP:journals/logcom/CaminadaCD12}
M.~W.~A. Caminada, W.~A. Carnielli, and P.~E. Dunne.
\newblock {Semi-stable semantics}.
\newblock {\em J. Log. Comput.}, 22(5):1207--1254, 2012.

\bibitem{DBLP:conf/kr/CayrolSL08}
C. Cayrol, F.~D. de~Saint{-}Cyr, and M. Lagasquie{-}Schiex.
\newblock {Revision of an Argumentation System}.
\newblock In G. Brewka and J. Lang, editors, {\em {Principles of Knowledge
  Representation and Reasoning: Proceedings of the Eleventh International
  Conference, {KR} 2008, Sydney, Australia, September 16-19, 2008}}, pages
  124--134. {AAAI} Press, 2008.

\bibitem{cramer2018directionality}
M. Cramer and M. Guillaume.
\newblock {Directionality of attacks in natural language argumentation}.
\newblock In C. Schon, editor, {\em {Proceedings of the Workshop on Bridging
  the Gap between Human and Automated Reasoning}}, volume 2261, pages 40--46.
  RWTH Aachen University, CEUR-WS.org, 2018.
\newblock \url{http://ceur-ws.org/Vol-2261/}.

\bibitem{cramer2018empirical}
M. Cramer and M. Guillaume.
\newblock {Empirical Cognitive Study on Abstract Argumentation Semantics}.
\newblock {\em Frontiers in Artificial Intelligence and Applications}, pages
  413--424, 2018.

\bibitem{cramer2019empirical}
M. Cramer and M. Guillaume.
\newblock {Empirical Study on Human Evaluation of Complex Argumentation
  Frameworks}.
\newblock In {\em {Proceedings of JELIA 2019}}, 2019.
\newblock Full paper available at
  \url{http://icr.uni.lu/mcramer/downloads/2019\_JELIA.pdf}.

\bibitem{dung1995acceptability}
P.~M. Dung.
\newblock {On the acceptability of arguments and its fundamental role in
  nonmonotonic reasoning, logic programming and n-person games}.
\newblock {\em Artificial Intelligence}, 77(2):321--357, 1995.

\bibitem{Dvorak16}
W. Dvo\v{r}{\'a}k and S.~A. Gaggl.
\newblock {Stage semantics and the SCC-recursive schema for argumentation
  semantics}.
\newblock {\em Journal of Logic and Computation}, 26(4):1149--1202, Aug 2016.

\bibitem{DBLP:journals/ai/OikarinenW11}
E. Oikarinen and S. Woltran.
\newblock {Characterizing strong equivalence for argumentation frameworks}.
\newblock {\em Artificial Intelligence}, 175(14-15):1985--2009, 2011.

\bibitem{rahwan2010behavioral}
I. Rahwan, M.~I. Madakkatel, J.-F. Bonnefon, R.~N. Awan, and S. Abdallah.
\newblock {Behavioral Experiments for Assessing the Abstract Argumentation
  Semantics of Reinstatement.}
\newblock {\em Cognitive Science}, 34(8):1483--1502, 2010.

\bibitem{rahwan2009argumentation}
I. Rahwan and G.~R. Simari.
\newblock {\em {Argumentation in Artificial Intelligence}}.
\newblock Springer Publishing Company, Incorporated, 1st edition, 2009.

\bibitem{DBLP:conf/tafa/RienstraST15}
T. Rienstra, C. Sakama, and L.~W.~N. van~der Torre.
\newblock {Persistence and Monotony Properties of Argumentation Semantics}.
\newblock In E. Black, S. Modgil, and N. Oren, editors, {\em {Theory and
  Applications of Formal Argumentation -- Revised Selected Papers}}, volume
  9524 of {\em {Lecture Notes in Computer Science}}, pages 211--225. Springer,
  2015.

\bibitem{vanderTorre18}
L. van~der Torre and S. Vesic.
\newblock {The principle-based approach to abstract argumentation semantics}.
\newblock In P. Baroni, D. Gabbay, M. Giacomin, and L. van~der Torre, editors,
  {\em {Handbook of Formal Argumentation}}. College Publications, 2018.

\bibitem{verheij1996two}
B. Verheij.
\newblock {Two Approaches to Dialectical Argumentation: Admissible Sets and
  Argumentation Stages}.
\newblock In {\em {In Proceedings of the biannual International Conference on
  Formal and Applied Practical Reasoning (FAPR) workshop}}, pages 357--368.
  Universiteit, 1996.

\end{thebibliography}
\end{document}